\documentclass{article}
\usepackage{fullpage}
\usepackage{graphicx}
\usepackage{natbib}
\usepackage{amsmath,amssymb,amsthm}
\usepackage{subfigure}
 
\usepackage{hyperref}
\usepackage{algorithm}
\usepackage[noend]{algcompatible}
\usepackage[all]{xy}
\newtheorem{theorem}{Theorem}

\usepackage[utf8]{inputenc} 
\usepackage{hyperref}       
\usepackage{url}            
\usepackage{booktabs}       
\usepackage{amsfonts}       
\usepackage{nicefrac}       
\usepackage{microtype}      
\usepackage{xcolor}         

\title{Efficient line search for optimizing Area Under the ROC Curve in gradient descent}

\author{%
  Jadon Fowler \\
  \texttt{j@jadon.io}\\
\\
  Toby Dylan Hocking \\
  Département d'Informatique \\
  Université de Sherbrooke \\
  \texttt{toby.dylan.hocking@usherbrooke.ca} \\
}

\begin{document}

\maketitle

\begin{abstract}
Receiver Operating Characteristic (ROC) curves are useful for evaluation in binary classification and changepoint detection, but difficult to use for learning since the Area Under the Curve (AUC) is piecewise constant (gradient zero almost everywhere).
Recently the Area Under Min (AUM) of false positive and false negative rates has been proposed as a differentiable surrogate for AUC.
In this paper we study the piecewise linear/constant nature of the AUM/AUC, and propose new efficient path-following algorithms for choosing the learning rate which is optimal for each step of gradient descent (line search), when optimizing a linear model.
Remarkably, our proposed line search algorithm has the same log-linear asymptotic time complexity as gradient descent with constant step size, but it computes a complete representation of the AUM/AUC as a function of step size.
In our empirical study of binary classification problems, we verify that our proposed algorithm is fast and exact; in changepoint detection problems we show that the proposed algorithm is just as accurate as grid search, but faster.
\end{abstract}
\section{Introduction}
\label{sec:introduction}
In supervised machine learning problems such as binary classification \citep{cortes2004auc} and changepoint detection \citep{Hocking2013icml}, the goal is to learn a function that is often evaluated using a Receiver Operating Characteristic (ROC) curve, which is a plot of True Positive Rate (TPR) versus False Positive Rate (FPR)\citep{egan1975signal}.
For data with $n$ labeled examples, a predicted value $\hat y_i\in\mathbb R$ is computed for each labeled example $i\in\{1,\dots,n\}$, and in binary classification the threshold of zero is used to classify as either positive ($\hat y_i>0$, True Positive=TP for a positive label, False Positive=FP for a negative label) or negative ($\hat y_i<0$, True Negative=TN for a negative label, False Negative=FN for a positive label).
Computing overall true positive and false positive rates yields a single point in ROC space, and the different points on the ROC curve are obtained by adding a real-valued constant $c\in\mathbb R$ to each predicted value $\hat y_i$ (Figure~\ref{fig:roc-curves}).
Large constants $c$ result in FPR=TPR=1 and small constants result in FPR=TPR=0; 
a perfect binary classification model has an AUC of 1, and a constant/bad model has an AUC of 0.5.

While AUC is often used as the evaluation metric in machine learning, it can not be used to compute gradients, because it is a piecewise constant function of predicted values.
Recently, \citet{hillman2023optimizing} proposed the AUM, or Area Under Min(FP,FN), as a surrogate loss, and showed that minimizing the AUM results in AUC maximization, in unbalanced binary classification and supervised changepoint detection.
More specifically, minimizing the AUM encourages points on the ROC curve to move to the upper left (Figure~\ref{fig:roc-curves}). 
In this paper, we propose a new gradient descent learning algorithm, which uses the gradients of the AUM loss, with a line search for either minimizing AUM or maximizing AUC (on either the subtrain or validation set).

\subsection{Contributions and organization}

Our main contribution is a new log-linear algorithm for efficiently computing a line search to determine 
the optimal step size (learning rate) in each step of gradient descent, when learning linear model with the AUM loss.
In Section~\ref{sec:model}, we define the AUM line search problem, then in Section~\ref{sec:algorithms}, we prove efficient update rules for computing changes in AUM and AUC, which results in a complete representation of the piecewise linear/constant AUM/AUC as a function of step size.
Remarkably, even though AUC can not be used to compute gradients (since it is piecewise constant), we show that it can be used as the criterion to maximize (on either the subtrain or validation set) during the computationally efficient log-linear line search algorithm.
In Section~\ref{sec:results}, we provide an empirical study of supervised binary classification and changepoint detection problems, showing that our new line search algorithm is fast (sub-quadratic), and just as accurate as a standard/slow grid search.
Section~\ref{sec:discussion} concludes with a discussion of the significance and novelty of our findings.

\begin{figure*}[t]
\vskip 0.2in
\begin{center}
\includegraphics[width=0.8\textwidth]{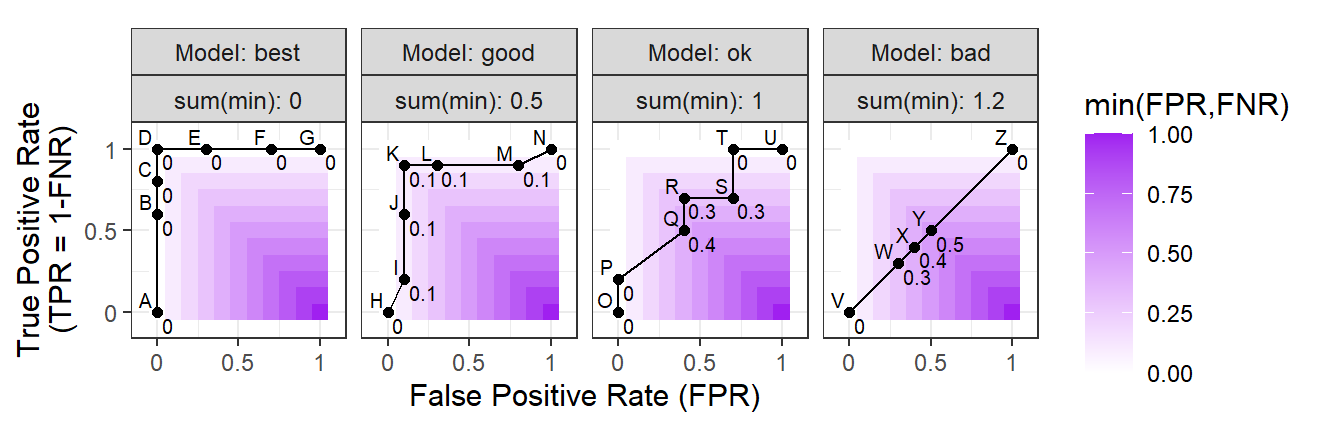}

\hskip 0.45cm
\includegraphics[width=0.8\textwidth]{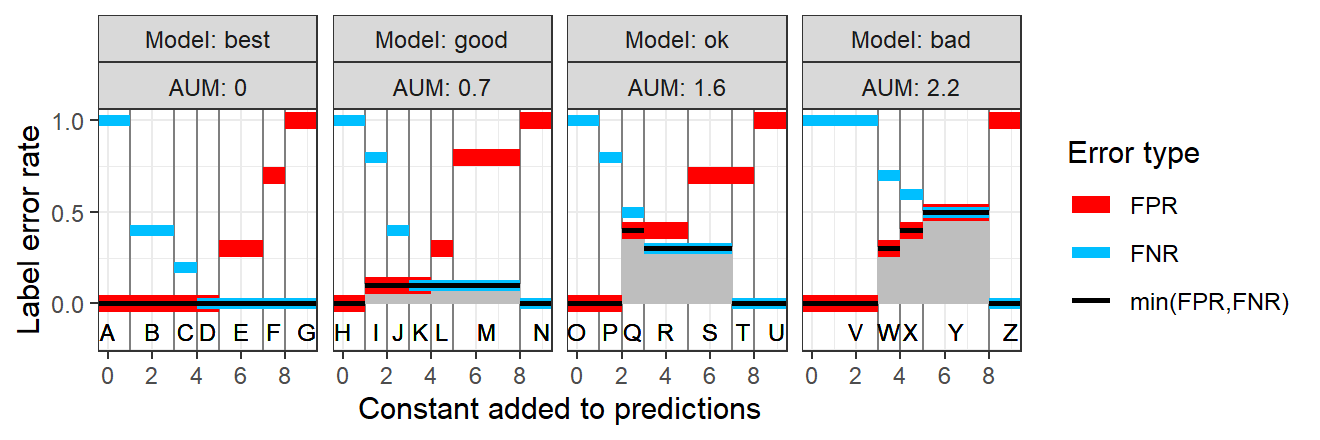}
\vskip -0.4cm
\caption{For four binary classification models,
there is one letter, A--Z, for each ROC point (top), and corresponding interval of constants added to predictions (bottom).
Number next to each ROC point shows min(FPR,FNR)
(same as purple heat map values, and black curve in bottom plot), which is minimal (0) when AUC is maximal (1).
The proposed algorithm is for minimizing the AUM, Area Under Min(FPR,FNR) (grey shaded region in bottom plot), which is a differentiable surrogate for the sum of min(FPR,FNR) over all points on the ROC curve (sum(min) values shown in top panel titles).
}
\label{fig:roc-curves}
\end{center}
\vskip -0.2in
\end{figure*}

\section{Related work}
\label{sec:related-work}

There are two groups of approaches to dealing with class imbalance in binary classification: re-weighting the typical logistic/cross-entropy loss, and using other loss functions (typically defined on pairs of positive and negative examples).
Re-weighting approaches include under-sampling and over-sampling \citep{fernandez2018smote}, and modifying the typical logistic/cross-entropy loss to account for expected level of balance/imbalance in the test set \citep{Menon2013}.  
Several algorithms are based on the idea of alternative pairwise loss functions, including some based on a surrogate like the squared hinge summed over pairs \citep{yuan2020auc,yuan2023libauc,rust2023log}.
Our proposed algorithm is based on the AUM loss \citep{hillman2023optimizing}, which can be categorized as an alternative loss function.
Although it is not defined on pairs, it instead requires a sorting of predicted values (similar to ROC curve computation).

Our proposed line search algorithm is similar to the idea of path-following (homotopy) algorithms, which are a popular topic in the statistical machine learning research literature. 
Classic homotopy algorithms include the LARS algorithm for the LASSO and the SVMpath algorithm for the Support Vector Machine, which compute the piecewise linear paths of optimal model parameters as a function of regularization parameter \citep{efron2004least,Dai2013}. 
Similar path-following algorithms include the fused lasso for segmentation and clusterpath for convex clustering \citep{hoefling2010path, hocking2011clusterpath}.
Whereas these algorithms compute a complete path of optimal solutions as a function of the regularization parameter, our proposed algorithm computes a complete path of AUC/AUM values as a function of the step size (learning rate) in an iteration of gradient descent. 

\section{Models and Definitions}
\label{sec:model}

\paragraph{In supervised binary classification,} we are given a set of $n$ labeled training examples, $\{(\mathbf x_i, y_i)\}_{i=1}^n$ where $\mathbf x_i\in\mathbb R^p$ is an input feature vector for one example and $y_i\in\{-1,1\}$ is a binary output/label.
The goal of binary classification is to learn a function $f:\mathbb R^p\rightarrow \mathbb R$ which is used to compute real-valued predictions $\hat y_i=f(\mathbf x_i)$ with the same sign as the corresponding label $y_i$.
Whereas typically the logistic/cross-entropy loss is used for learning $f$, our proposed algorithm uses the AUM loss \citep{hillman2023optimizing}. 

\begin{figure*}[t]
\vskip 0.2in
\begin{center}
\includegraphics[height=1.7in]{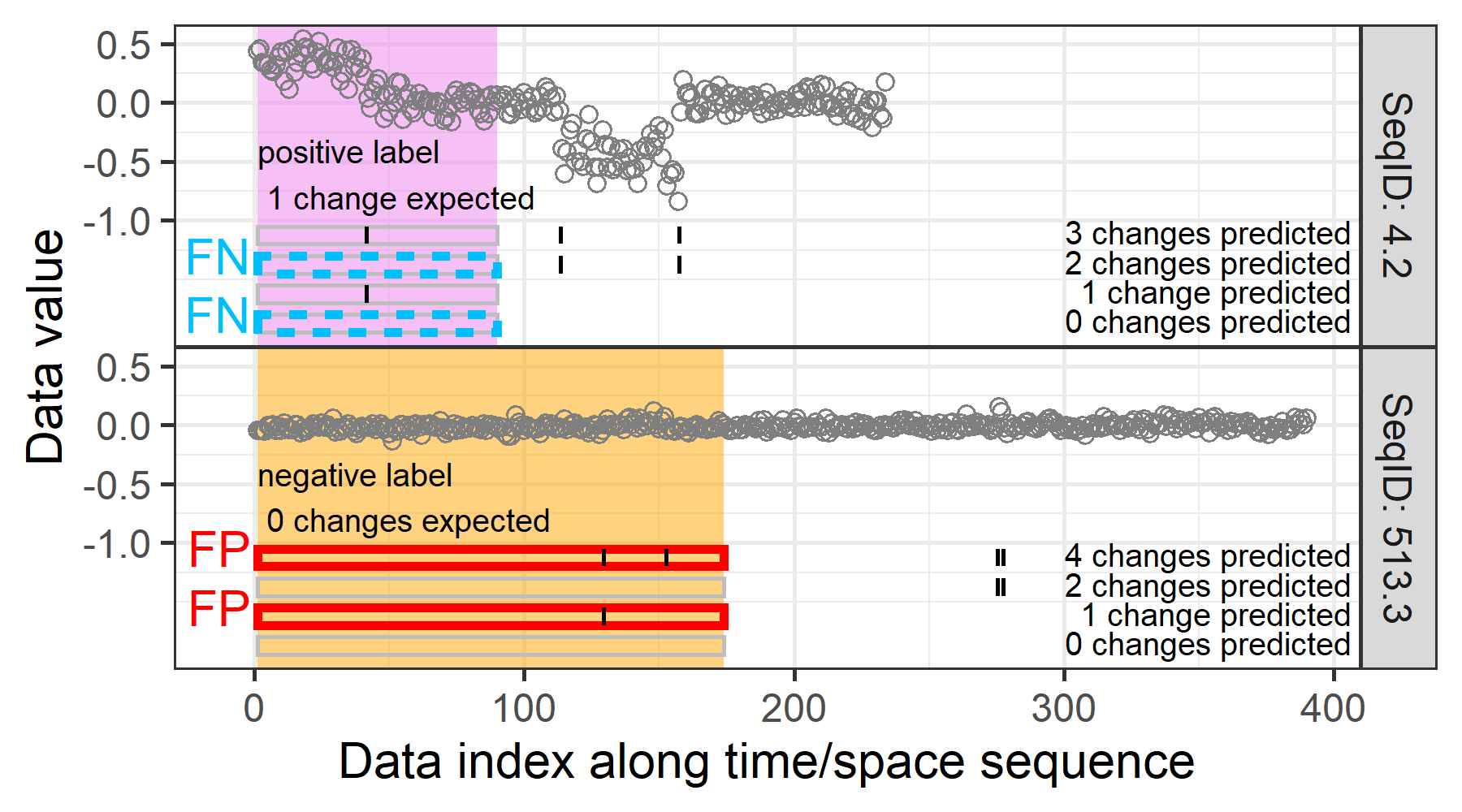}
\includegraphics[height=1.7in]{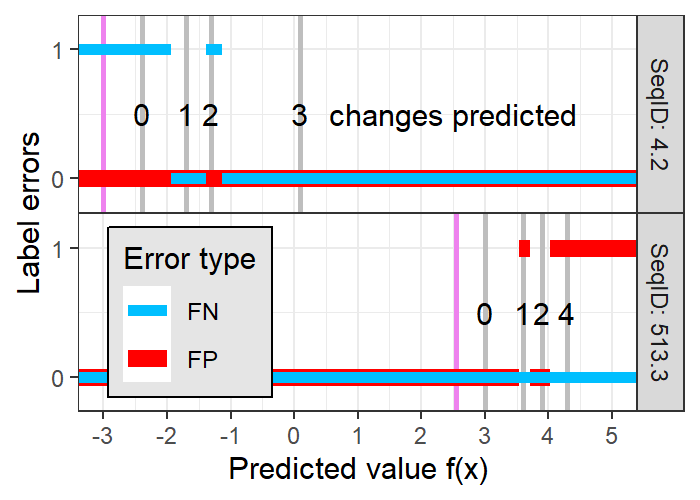}
\vskip -0.4cm
\caption{Two labeled changepoint problems (left), and corresponding error functions (right).
In these two problems, the FN/FP error functions are non-monotonic, because a changepoint disappears when moving from model size 1 to 2.
Vertical purple lines (right) mark predicted values which result in the line search shown in Figure~\ref{fig:line-search-example}.
}
\label{fig:two-labeled-changepoint}
\end{center}
\vskip -0.2in
\end{figure*}

\paragraph{In supervised changepoint detection,} we assume for each labeled training example $i\in\{1,\dots,n\}$ there is a corresponding data sequence vector $\mathbf z_i$ and label set $L_i$ \citep{Hocking2013icml}. 
For example in Figure~\ref{fig:two-labeled-changepoint} we show two data sequences, each with a single label.
Dynamic programming algorithms can be used on the data sequence $\mathbf z_i$ to compute optimal changepoint models for different numbers of changepoints $\{0,1,\dots\}$  \citep{Maidstone2016}.
For example in Figure~\ref{fig:two-labeled-changepoint} (left) we show models with 0--4 changepoints.
The label set $L_i$ can be used to compute the number of false positive and false negative labels with respect to any predicted set of changepoints (false positives for too many changepoints, false negatives for not enough changepoints).
Each example $i$ also has a model selection function $\kappa^*_i:\mathbb R^+_0 \rightarrow \{0,1,\dots\}$ which maps a non-negative penalty value $\hat \lambda_i$ to a selected number of changepoints $\kappa^*_i(\hat \lambda_i)$, and corresponding FP/FN error values (Figure~\ref{fig:two-labeled-changepoint}, right).
We assume there is a fixed feature map $\phi$ which can be used to compute a feature vector $\mathbf x_i = \phi(\mathbf z_i)\in\mathbb R^p$ for each labeled example.
We want to learn a function $f:\mathbb R^p\rightarrow \mathbb R$ which inputs a feature vector and outputs a real-valued prediction that is used as a negative log penalty value, $f(\mathbf x_i) = -\log \hat \lambda_i$.
The goal is to predict model sizes $\kappa^*_i(\hat \lambda_i)$ that result in minimal label errors or AUC=1.
Note the FP/FN label error functions may be non-monotonic (Figure~\ref{fig:two-labeled-changepoint}, right), which means that the ROC curve may be non-monotonic, with loops/cycles, and AUC outside the typical range of [0,1], such as AUC=2 shown in Figure~\ref{fig:line-search-example} \citep{hillman2023optimizing}.
Whereas typical loss functions used for learning $f$ are based on regression with interval censored outputs \citep{Hocking2013icml,Drouin2017,barnwal2021aftxgboost}, our proposed algorithm uses the AUM loss \citep{hillman2023optimizing}.

\subsection{Definition of false positive and false negative functions}
\label{sec:def-fp-fn}

In this paper, we assume the following general learning context in which supervised binary classification and changepoint detection are specific examples. 
For each labeled training observation $i$, we have a predicted value $\hat y_i=f(\mathbf x_i)\in\mathbb R$, and one or more labels which let us compute the contribution of this observation to the false negative rate $\text{FN}_i(\hat y_i)\in[0,1]$ and false positive rate $\text{FP}_i(\hat y_i)\in[0,1]$ (for example, Figure~\ref{fig:two-labeled-changepoint}, right). 
The $\text{FP}_i$ starts at zero (no false positives for very small predicted values), the $\text{FN}_i$ ends at zero (no false negatives for very large predictive values).
These functions are piecewise constant, so can be represented by breakpoint tuples $(v,\Delta\text{FP},\Delta\text{FN})$, where $v\in\mathbb R$ is a predicted value threshold where there are changes $\Delta\text{FP},\Delta\text{FN}$ (discontinuity) in the error functions.
In binary classification with $n^+$ positive examples and $n^-$ negative examples, these functions can be exactly represented by the breakpoint
$(v=0,\Delta\text{FP}=0,\Delta\text{FN}=-1/n^+)$ for all positive examples, and 
$(v=0,\Delta\text{FP}=1/n^-,\Delta\text{FN}=0)$ for all negative examples.
In supervised changepoint detection, there can be more than one breakpoint per error function (for example, Figure~\ref{fig:two-labeled-changepoint}, right, shows three breakpoints per error function).
These breakpoints will be used in our proposed learning algorithm, since they give information about how the predicted values affect the ROC curve.

\begin{figure*}[t]
\begin{center}
\parbox{0.6\textwidth}{
\hspace{0.0cm}
\includegraphics[width=0.5\textwidth]{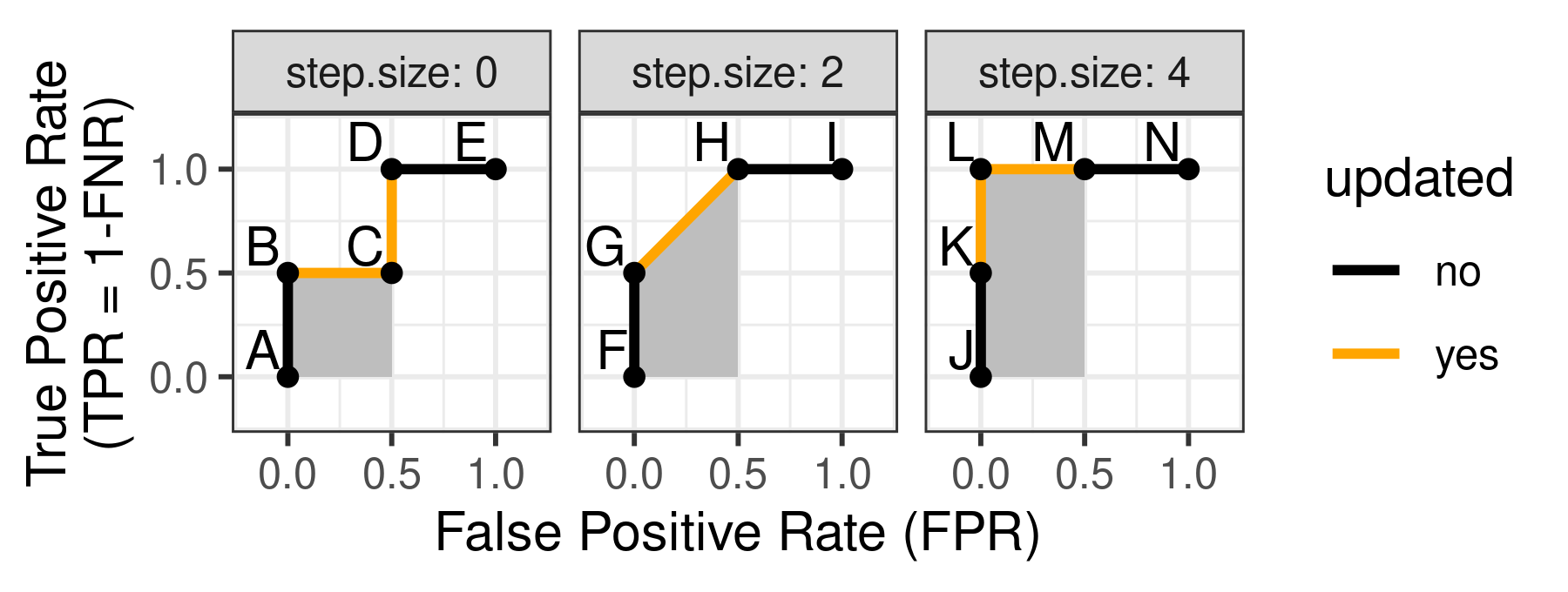}
\\
\includegraphics[width=0.6\textwidth]{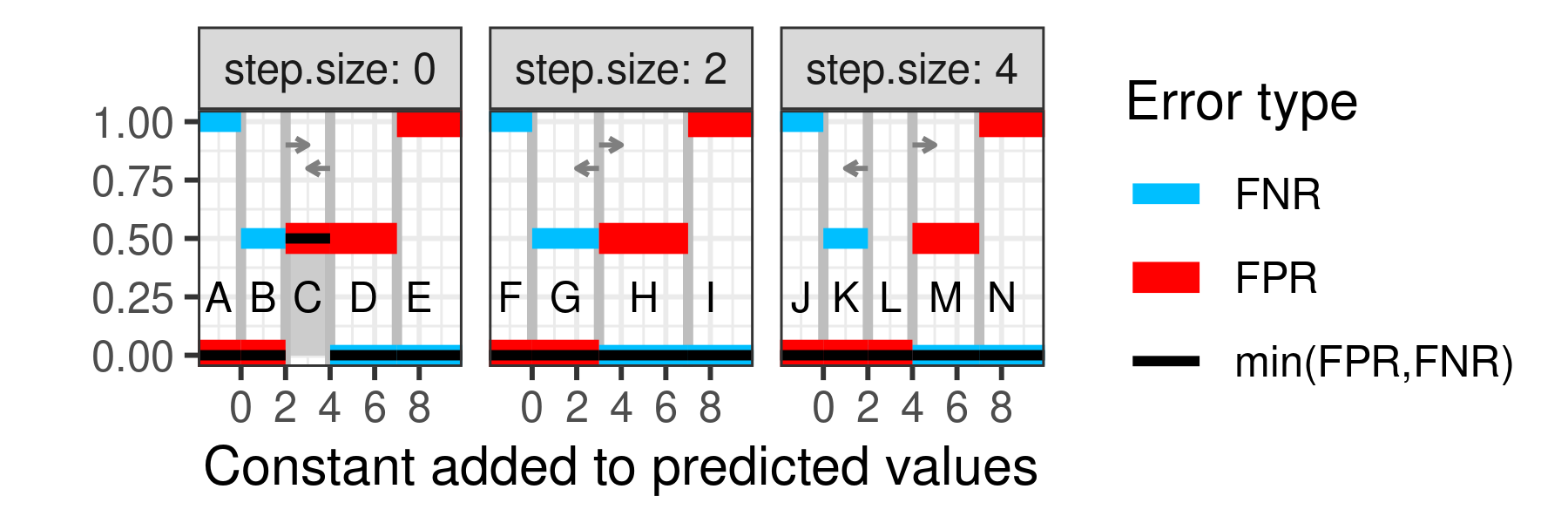}
}\parbox{0.39\textwidth}{
\includegraphics[width=0.35\textwidth]{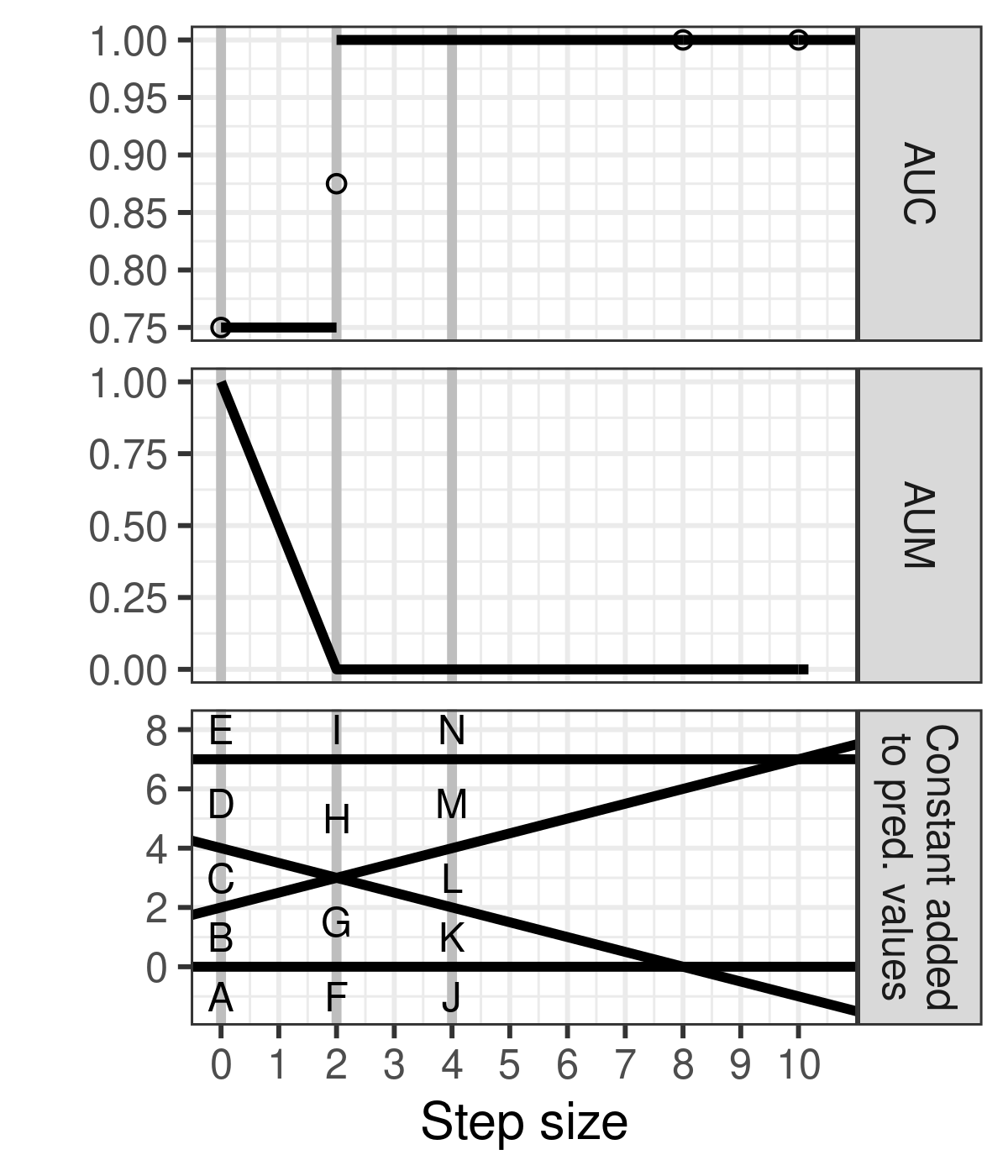}
}
\vskip -0.3cm
\caption{
Demonstration of proposed line search algorithm, for a simple binary classification problem with four data.
\textbf{Top left:} ROC curves at three step sizes, with shaded grey area showing parts of AUC involved in the update rules (\ref{eq:AUC_without}--\ref{eq:AUC_after}).
\textbf{Bottom left:} error rate functions at three step sizes, with grey arrows showing the gradient, and shaded grey area (C) showing the AUM, Area Under Min(FPR,FNR). 
\textbf{Right:} AUC, AUM, and threshold functions $T_b(s)$ (black lines), as a function of step size.
There is one letter for every ROC point, corresponding to an interval of constants added to predicted values at a given step size.
}
\label{fig:line-search-example-binary-roc}
\end{center}
\vskip -0.2in
\end{figure*}



\subsection{Linear model predictions and errors as a function of step size}

Let there be a total of $B$ breakpoints in the error functions over all $n$ labeled training examples, where each breakpoint $b\in\{1,\dots, B\}$ is represented by the tuple $(v_b, \Delta\text{FP}_b, \Delta\text{FN}_b, \mathcal I_b)$.
The $\mathcal I_b\in\{1,\dots,n\}$ is an example index, so there are changes $\Delta\text{FP}_b, \Delta\text{FN}_b$ at predicted value $v_b\in\mathbb R$ in the error functions $\text{FP}_{\mathcal I_b},\text{FN}_{\mathcal I_b}$.
For example the labeled data sequences shown in Figure~\ref{fig:two-labeled-changepoint} represent $n=2$ labeled training examples, with $B=6$ breakpoints total.
For a linear model $f(\mathbf{x}_i)=\mathbf w^T \mathbf{x}_i$,
we can compute a descent direction $\mathbf d\in \mathbb R^p$ based on the directional derivatives of the AUM \citep{hillman2023optimizing}, which gives the following predictions, as a function of step size $s$ (learning rate), 
\begin{equation}
    \hat y_i(s)=(\mathbf w+s\mathbf d)^T \mathbf{x}_i.
\end{equation}
For each breakpoint $b$, we define the following function, which tells us how its threshold evolves as a function of step size, in the plot of error rates (Figures \ref{fig:roc-curves} and \ref{fig:line-search-example-binary-roc}) as a function of the constant $c$ added to predicted values in ROC curve computation.
\begin{equation}
    T_b(s) = v_b - \hat y_{\mathcal I_b}(s) =
    v_b - (\mathbf w+s\mathbf d)^T \mathbf x_{\mathcal I_b}.
\end{equation}
We can see from the equation above that $T_b(s)$ is a linear function with slope $\delta_b = - \mathbf d^T \mathbf x_{\mathcal I_b}$ and intercept $\epsilon_b=v_b - \hat y_{\mathcal I_b}(0) = v_b - \mathbf w^T \mathbf x_{\mathcal I_b}$.
This equation can be used to plot the threshold $T_b(s)$ as a function of the step size $s$ (Figures~\ref{fig:line-search-example-binary-roc}--\ref{fig:line-search-example}).
Given the $B$ observation error breakpoints, a prediction vector $\mathbf{\hat y}=[\hat y_1 \cdots \hat y_n]^\intercal\in\mathbb R^n$, and a descent direction $\mathbf d$, we define the error functions
\begin{equation}
  \underline{\text{FP}}_b(s) = \sum_{j: T_j(s) < T_b(s)} \Delta\text{FP}_j,\ \ \ \ \ \ \ \  \label{eq:first-fp-fn-over-under} 
  \underline{\text{FN}}_b(s) = \sum_{j: T_j(s) \geq T_b(s)} - \Delta\text{FN}_j.
\end{equation}
The $\underline{\text{FP}}_b(s), \underline{\text{FN}}_b(s)\in[0,1]$ are the error rates before the threshold $T_b(s)$, in the plot of label error rates as a function of the constant $c$ added to predicted values (for example, Figure~\ref{fig:line-search-example-binary-roc}, bottom left). 
Additionally we define $\underline{M}_b(s) = \min\{ 
    \underline{\text{FP}}_b(s),
    \underline{\text{FN}}_b(s)
    \}$ which is the min of the two error rates.
\section{Proposed line search algorithm}
\label{sec:algorithms}

In this section, we provide theorems which state the update rules of our proposed algorithm.

\subsection{Initialization}

\begin{figure*}[t]
\vskip 0.2in
\begin{center}
\parbox{0.65\textwidth}{
\includegraphics[width=0.65\textwidth]{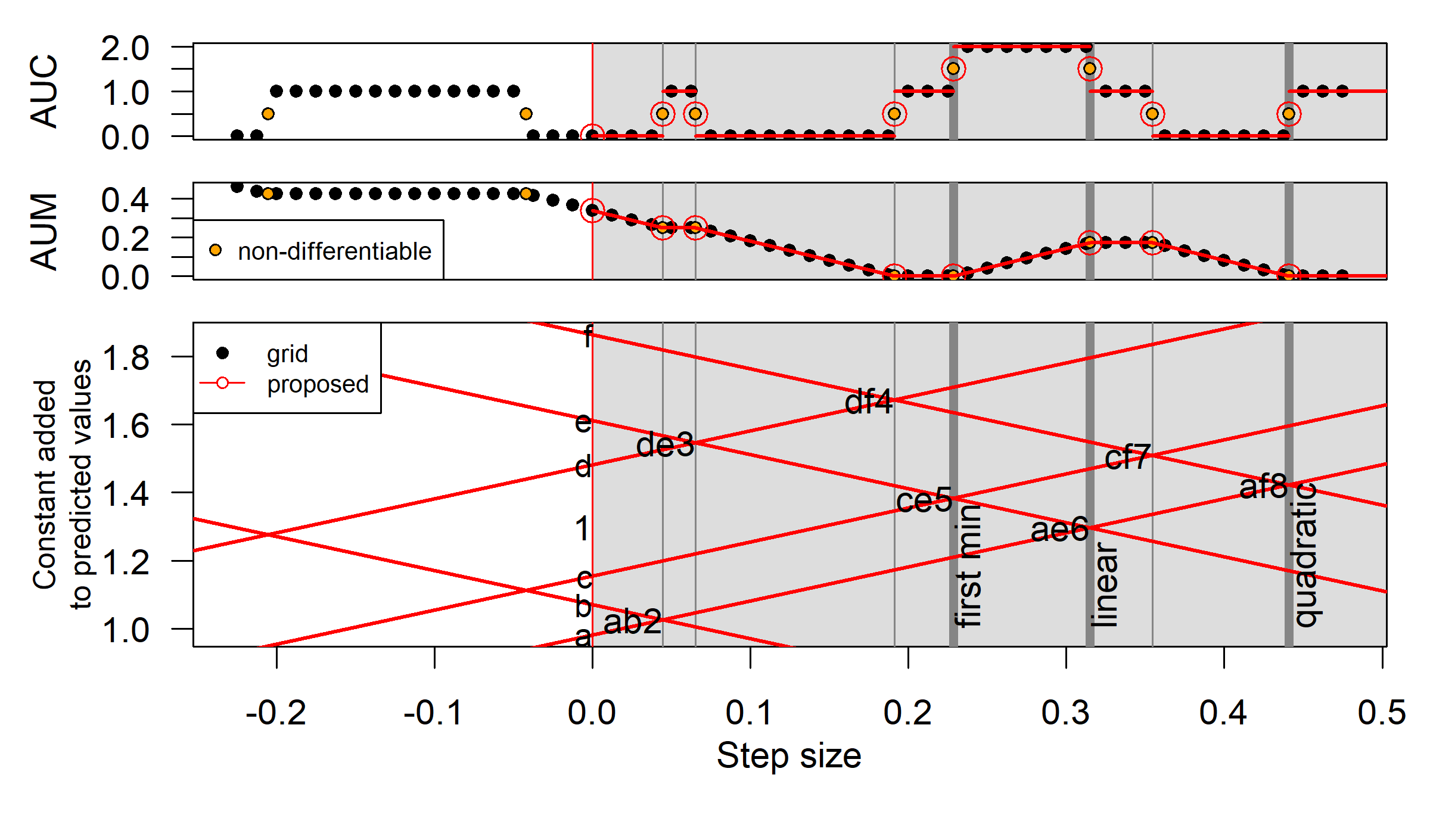}
}
\parbox{0.155\textwidth}{
Red-black tree of possible next intersections:
\vskip 0.1cm
\begin{tabular}{rl}
  \hline
it & next \\ 
  \hline
  1 & ab,de \\ 
    2 & de \\ 
    3 & df,ce \\ 
    4 & ce \\ 
    5 & ae,cf \\ 
    6 & cf \\ 
    7 & af \\ 
    8 &  \\ 
   \hline
\end{tabular}
}
\vskip -0.5cm
\caption{
Demonstration of proposed line search algorithm, for the same two labeled changepoint data sequences as in Figure~\ref{fig:two-labeled-changepoint}.
It starts by computing AUM/AUC at step size 0 (vertical red line, iteration 1), and storing the next possible intersection points in a red-black tree (right table).
Iteration 2 removes the intersection point with the smallest step size (ab), resulting in a change of AUM slope (from -2 to 0), and a change of AUC values (from 0 to 0.5 at the intersection point, then to 1 after), and no new intersection points.
Three vertical grey lines represent variants with different stopping rules: first min is the smallest step size such that AUM would increase for larger step sizes, 
linear is the same number of iterations as the number of red lines in the threshold plot (6 lines: a--f), and quadratic means to explore all positive step sizes (shaded grey area).
Note AUC can be larger than 1 because there are cycles/loops in the ROC curve, due to non-monotonic label error functions.
}
\label{fig:line-search-example}
\end{center}
\vskip -0.2in
\end{figure*}

Let $K>1$ be the max number of iterations, and let $k\in\{1,2,\dots,K\}$ be the counter for the iterations of our proposed algorithm.
For each iteration $k$, there is a step size $\sigma_k$, and the initial step size is $\sigma_1=0$.
The proposed algorithm computes an exact representation of the piecewise constant error rates (FPR,FNR) as a function of the constant $c$ added to predicted values (same as in ROC curve computation, see Figure~\ref{fig:roc-curves}). 
At each step size, there are $B$ error rates (FPR,FNR) which must be computed, one for each breakpoint $b\in\{2,\dots,B\}$ in label error functions.
We use the notation $\text{FP}_b^k$ to denote the false positive rate before breakpoint $b$, at iteration $k$.
Note that we use the superscript $k$ to clarify the presentation of the update rules in this paper, but the algorithm only stores values for the current $k$, using $O(B)$ storage. 
Let $\mathbf q^1 = (q^1_1,\dots,q^1_B)$ be a permutation of $(1,\dots,B)$ such that the threshold functions are sorted, 
$T_{q^1_1}(0)\leq \cdots \leq T_{q^1_B}(0)$.
We also have for all $b\in\{2,\dots,B\}$, initialize
$\text{FP}^1_b=\underline{\text{FP}}_{q^1_b}(0)$ and 
$\text{FN}^1_b=\underline{\text{FN}}_{q^1_b}(0)$ (also $\text{TP}^1_b=1-\text{FN}^1_b$ and 
$M^1_b=\min\{\text{FP}^1_b, \text{FN}^1_b\}$) which is possible in log-linear time, by first sorting by threshold values, $T_{q_b^1}(0)$, then using a cumulative sum (\ref{eq:first-fp-fn-over-under}).
Note that these initializations start at index 2 and end at index $B$; the first index is missing because the Min below the first interval is always zero (by assumption that the False Positive Rate starts at zero).
Similarly, the Min after the last interval is always zero, by assumption that the False Negative Rate ends at zero. 
That is, for any iteration $k$, we have $M_1^k=0$ and $M_{B+1}^k=0$.
In the first iteration, we compute the AUC at step size 0 using the trapezoidal area, (area of triangle + area of rectangle under each segment of the ROC curve)
\begin{equation}
    \text{AUC}^1_{\text{after}} = 
    \sum_{b=1}^{B}
    (\text{FP}^1_{b+1}-\text{FP}^1_b)
    (\text{TP}^1_{b+1}+\text{TP}^1_b)/2.
    \label{eq:AUC1}
\end{equation}
Also, we compute the initial AUM via
\begin{equation}
\label{eq:AUM_init}
    \text{AUM}^1 = \sum_{b=2}^B [T_{q_b^1}(0) - T_{q_{b-1}^1}(0)] M^1_b,
\end{equation}
and its initial slope as a function of step size is
\begin{eqnarray}
  D^1 &=& \sum_{b=2}^B [\delta_{q^1_b} - \delta_{q_{b-1}^1}] M^1_b. \label{eq:D1}
\end{eqnarray}
\subsection{Update rules for error functions}
The initialization is valid for any step sizes $s\in(0,\sigma_2)$, where $\sigma_2$ is the smallest step size such that the permutation $\mathbf q$ is no longer a valid ordering of the threshold functions (there is an intersection between two or more threshold functions $T_b$ at $\sigma_2$).
More generally, for any iteration $k\in\{2,\dots,\}$, we assume that at step size $\sigma_k$, there is an intersection, $T_{q^{k}_\beta}(\sigma_k) = T_{q^{k}_{\beta-1}}(\sigma_k)$, and $\beta$ is the index of the function which is larger before the intersection. 
Then the update for the permutation is
for all $b\in\{1,\dots,B\}$
\begin{equation}
    q^{k}_b = \begin{cases}
    q^{k-1}_{b-1} &\text{ if } b=\beta,\\
    q^{k-1}_{b+1} &\text{ if } b=\beta-1,\\
    q^{k-1}_b & \text{ otherwise.}
    \end{cases}
\end{equation}
For any iteration $k\in\{2,\dots\}$ with intersection point at step size $\sigma_k$, this update gives a new permutation $\mathbf q^k = (q^k_1,\dots,q^k_B)$ such that for all $s\in[\sigma_{k}, \sigma_{k+1}]$ we have
\begin{eqnarray}
T_{q^k_1}(s)\leq &\cdots& \leq T_{q^k_B}(s) \label{eq:T_order}
\end{eqnarray}
Above (\ref{eq:T_order}) means that the permutation $\mathbf q^k$ results in the $T_b(s)$ functions being sorted for all step sizes $s\in[\sigma_{k},\sigma_{k+1}]$ before the next intersection point $\sigma_{k+1}$.
We would like to compute AUM for any $s\in[\sigma_{k},\sigma_{k+1}]$ via a linear function,
\begin{equation}
\label{eq:AUM_s}
    A(s) = \sum_{b=2}^B [T_{q_b^k}(s) - T_{q_{b-1}^k}(s)] M^k_b = \text{AUM}^k + sD^k.
\end{equation}
The line search problem is to minimize this function, $\min_{s>0} A(s)$.
For any $s\in[\sigma_{k},\sigma_{k+1}]$, $A(s)$ is a linear function with intercept is $\text{AUM}^k$ and slope $D^k$.
The main insight of our algorithm is that there is a constant $O(1)$ time update rule for computing slopes $D^k$ at subsequent iterations $k\in\{2,\dots\}$.
The algorithm must keep track of FP/FN vectors (of size $B-1$), which can be updated for all $b\in\{2,\dots,B\}$ via ($I$ is the indicator function, 1 if argument true, else 0)
\begin{eqnarray}
  \text{FP}^{k+1}_b &=& \text{FP}^k_b + I[b=\beta] [\Delta\text{FP}_{q^k_\beta} - \Delta\text{FP}_{q^k_{\beta-1}}], \\
  \text{FN}^{k+1}_b &=& \text{FN}^k_b + I[b=\beta] [\Delta\text{FN}_{q^k_\beta} - \Delta\text{FN}_{q^k_{\beta-1}}]. 
\end{eqnarray}
The equation above says that the only entry that needs to be updated in the FP/FN vectors is $\beta$ (which was the index of the function which was larger before the intersection at step size $\sigma_k$).
After updating FP/FN the new min can also be efficiently computed for all $b\in\{2,\dots,B\}$ via
\begin{equation}
    M^{k+1}_b = \begin{cases}
    \min\{\text{FP}_b^{k+1}, \text{FN}_b^{k+1}\} &\text{ if } b=\beta,\\
    M^k_b & \text{ otherwise.}
    \end{cases}
\end{equation}
The equation above says that all the min values stay the same except entry $\beta$.

\subsection{Update rules for AUM and AUC}


Next, we state our first main result, the constant time update rule for the AUM slope $D^k$.
\begin{theorem}
For data with $B$ breakpoints in label error functions, the initial AUM slope is computed via~(\ref{eq:D1}) in log-linear $O(B\log B)$ time. 
If $\beta\in\{2,\dots,B\}$ is the index of the function $T_\beta$ which is larger before an intersection at step size $\sigma_{k+1}$, then the next AUM slope $D^{k+1}$ can be computed from the previous $D^{k}$ in constant $O(1)$ time, using~(\ref{eq:D_update}).
\begin{eqnarray}
    D^{k+1} &=& D^k 
    +\left(
    \delta_{q_\beta^k}- \delta_{q_{\beta-1}^k}
    \right)
    \left(
    I[\beta<B]M_{\beta+1}^k+ M_{\beta-1}^k
    -M_{\beta}^k - M_\beta^{k+1}
    \right).\label{eq:D_update}
\end{eqnarray}
\end{theorem}
\begin{proof}
The result can be derived by writing the terms in $D^{k+1}$ and $D^k$, then subtracting:
\begin{align}
\label{eq:D_diff}
    D^{k + 1} - D^k &= \left(\sum_{b=2}^B [\delta_{q^{k+1}_b} - \delta_{q_{b-1}^{k+1}}] M^{k+1}_b\right) - \left(\sum_{b=2}^B [\delta_{q^k_b} - \delta_{q_{b-1}^k}] M^k_b\right)\\
    &\ \ \vdots\nonumber\\
    \label{eq:D_update_term}
    &=  \left(\delta_{q_\beta^k} - \delta_{q_{\beta-1}^k}\right) \left(I[\beta<B]M_{\beta+1}^k +  M_{\beta-1}^k - M_{\beta}^k - M_\beta^{k+1}\right).
\end{align}
Adding $D^k$ to both sides of the above equation yields the desired result.
\end{proof}

Next, we state the update rules for the AUC.
The important idea behind the update rule is that when $T_b$ threshold functions intersect at step size $\sigma_{k+1}$, that corresponds to removing a point from the ROC curve (Figure~\ref{fig:line-search-example-binary-roc}, step size 0, subtract the corresponding area to get $\text{AUC}_{\text{without}}^{k+1}$), and replacing it with a diagonal connecting the adjacent points (Figure~\ref{fig:line-search-example-binary-roc}, step size 2, adding new area to get $\text{AUC}_{\text{at}}^{k+1}$).
Additionally, after the intersection point, there is a new ROC point that appears to replace the removed/old point, and connects to the adjacent points (Figure~\ref{fig:line-search-example-binary-roc}, step size 4, adding new area to get $\text{AUC}_{\text{after}}^{k+1}$).
\begin{theorem}
For data with $B$ breakpoints in label error functions, the initial AUC is computed via~(\ref{eq:AUC1}) in log-linear $O(B\log B)$ time. 
If $\beta\in\{2,\dots,B\}$ is the index of the function $T_\beta$ which is larger before intersection at step size $\sigma_{k+1}$, then the new AUC values can be computed in constant $O(1)$ time using (\ref{eq:AUC_without}--\ref{eq:AUC_after}).
\end{theorem}
\vskip -1cm
\begin{eqnarray}
\text{AUC}_{\text{without}}^{k+1} &=& 
    \text{AUC}_{\text{after}}^k -
    \sum_{b=\beta-1}^\beta
    (\text{FP}_{b+1}^k-\text{FP}_b^{k})
    (\text{TP}_{b+1}^k+\text{TP}_b^{k})/2.
    \label{eq:AUC_without}\\
    \text{AUC}_{\text{at}}^{k+1} &=& 
    \text{AUC}_{\text{without}}^{k+1} +
    (\text{FP}_{\beta+1}^k-\text{FP}_{\beta-1}^{k})
    (\text{TP}_{\beta+1}^k+\text{TP}_{\beta-1}^{k})/2.
    \label{eq:AUC_at}\\
    \text{AUC}_{\text{after}}^{k+1} &=& 
    \text{AUC}_{\text{without}}^{k+1} +
    \sum_{b=\beta-1}^\beta
    (\text{FP}_{b+1}^{k+1}-\text{FP}_b^{k+1})
    (\text{TP}_{b+1}^{k+1}+\text{TP}_b^{k+1})/2.
    \label{eq:AUC_after}
\end{eqnarray}
\begin{proof}
    The proof is analogous to the AUM update rule proof~(\ref{eq:D_diff}--\ref{eq:D_update_term}).
\end{proof}

\subsection{Implementation Details}

The update rules which we proposed in the previous section require identification of a pair of threshold evolution functions which are the next to intersect, $T_{q^k_b}(\sigma_k) = T_{q^k_{b-1}}(\sigma_k)$.
To efficiently perform this identification, we propose an algorithm which begins by sorting the linear $T_b$ functions by intercept, then using intercept/slope values to store intersections of all $B-1$ possible pairs of adjacent functions.
There may be fewer than $B-1$ intersections to store, because some adjacent pairs may have parallel lines, or intersection at negative step sizes.
Typically intersections involve only two lines, but when there are more, they can be handled using the following data structures:
\begin{itemize}
    \item An \textbf{Interval Group} is a collection of lines that intersect at the same step size and threshold.
    \item An \textbf{Interval Column} is an associative array where each key is an intersection threshold, and each value is an Interval Group. This contains all of the intersections at a given step size.
    \item A \textbf{Interval Queue} is a C++ STL map from step sizes to Interval Columns (red-black tree, Figure~\ref{fig:line-search-example}, right).
    The map container allows constant time lookup of the next intersection (Interval Column with smallest step size), and log-linear time insertion of new entries.
\end{itemize}
The algorithm starts by creating an Interval Queue and filling it with all of the intersections between every line and the line after it. Each iteration looks at the first Interval Column in the queue which at the start will be the one with the step size closest to 0. 
We update AUM slope and AUC using Theorems~1--2, and insert up to two new intersections, each of which takes $O(\log B)$ time using the STL map (red-black tree). 
We run this algorithm until we have reached the desired number of iterations, or we have found the first local min AUM, or first local max AUC.
Asymptotic complexity is $O(B)$ space and $O([B+I]\log B)$ time, where $I$ is the number of iterations.
Finally we note that the update rules can be implemented with respect to either the subtrain set (to guarantee decreasing AUM at every step), or with respect to the validation set (to search for max AUC that avoids overfitting).

\section{Empirical Results}

\citet{hillman2023optimizing} provided a detailed empirical study of the AUM loss relative to other baseline loss functions (logistic loss, re-weighting, squared hinge loss defined on all pairs of positive and negative examples), so the empirical study in the current paper focuses on characterizing the time complexity of the proposed line search algorithm.
No special computer/cluster was required for computational experiments.

\label{sec:results}



\begin{figure*}[t]
\vskip 0.2in
\begin{center}

\includegraphics[width=0.8\textwidth]{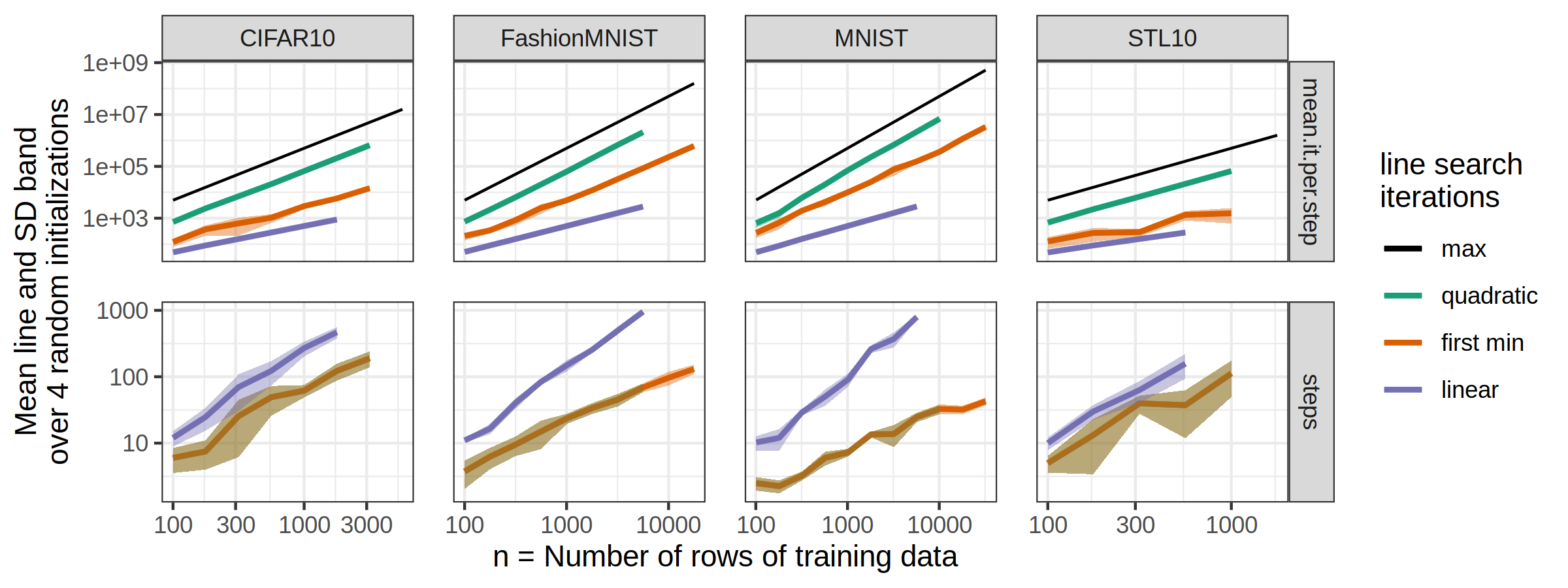}
\vskip -0.4cm
\caption{
Asymptotic time complexity of gradient descent with proposed line search in four binary classification data sets (CIFAR10, FashionMNIST, MNIST, STL10, first class versus others).
\textbf{Top:} number of line search iterations per gradient descent step is $O(n^2)$ in worst case (max);
exploring all intersections is $O(n^2)$ (quadratic);
exploring only the first $n$ is $O(n)$ (linear);
exploring until AUM increases (first min) is sub-quadratic but super-linear.
\textbf{Bottom:}
number of gradient descent steps until AUM stops decreasing (within $10^{-3}$); 
first min/quadratic methods (larger step sizes) take asymptotically fewer steps than linear (smaller step sizes).
}
\label{fig:binary-results}
\end{center}
\vskip -0.2in
\end{figure*}

\subsection{Empirical asymptotic time complexity analysis in benchmark classification data sets}
\paragraph{Goal and expectation.}
In this section, our goal was to empirically estimate the asymptotic complexity of the proposed line search, with the three proposed variants (linear, quadratic, first min, see Figure~\ref{fig:line-search-example}).
For a binary classification data set with $n$ labeled examples (and therefore $B=n$ breakpoints in label error functions), we expected that: 
(1/linear) line search with $n$ iterations should be log-linear time, $O(n\log n)$; 
(2/quadratic) line search exploring all intersections should be quadratic, $O(n^2)$;
(3/first min) line search exploring up until the first min AUM should be faster than quadratic (but guaranteed to find the same solution, due to the convexity of the AUM loss function).

\paragraph{Data and algorithm setup.}
We considered four benchmark binary classification data sets: CIFAR10 \citep{alex2009learning}, FashionMNIST \citep{xiao2017fashion},  MNIST \citep{lecun1998gradient},  STL10 \citep{STL10}.
For each data set, there are ten classes, so we converted them into an unbalanced binary classification problem by using the first class as negative label, and the other classes as positive label. 
For various data sizes $n$ starting with at least 10 examples of the minority class, and then increasing $n$, we randomly initialized the linear model near zero (four random seeds, standard normal), then implemented gradient descent with the three versions of AUM line search (full gradient method, batch size $n$), continuing until the AUM stops decreasing (within $10^{-3}$).

\paragraph{Experimental results.}
In Figure~\ref{fig:binary-results} (top), we show the mean number of line search iterations per gradient descent step, for each of the three line search variants, along with the maximum possible number of iterations for a given $n$ (black line, max intersections of $n$ lines is $n(n-1)/2$, a quadratic upper bound on the number of iterations/step sizes considered by our proposed line search).
Interestingly, the number of line search iterations of the first min variant appears to be sub-quadratic (smaller slope on log-log plot), indicating that performing an exact line search is computationally tractable, nearly log-linear $O(n\log n)$ (amortized average number of iterations over all steps of gradient descent).  
Also, we observed that the number of gradient descent steps for first min is asymptotically small for the first min variant (same as quadratic), whereas the linear variant is relatively large.
Overall, these results suggest that the proposed line search is fast and exact in benchmark binary classification data sets.

\subsection{Accuracy and computation time in supervised changepoint problems}
\paragraph{Motivation and setup.}
We were also interested to examine the accuracy of the line search, as measured by the max validation AUC over all steps of gradient descent.
We expected that the proposed line search should be faster than standard grid search, and be just as accurate.
We tested this expectation using the \texttt{chipseq} supervised changepoint data set from the UCI repository \citep{asuncion2007uci}.
In one representative supervised changepoint data set (H3K4me3\_TDH\_immune), we used 4-fold cross-validation to create train/test splits, then further divided the train set into subtrain/validation sets (four random seeds).
We initialized the linear model by minimizing a L1-regularized convex surrogate of the label error \citep{Hocking2013icml}, with L1 penalty chosen using the validation set, then ran AUM gradient descent using the proposed line search or grid search (using gradients from the subtrain set, batch size $n$), until AUM stops decreasing (within $10^{-5})$.

\paragraph{Experimental results.}
After every step of gradient descent, the AUC was computed on the validation set, and we report the max AUC in Figure~\ref{fig:result-four-changepoint-folds} (bottom).
It is clear that the proposed algorithms achieve a similar level of validation AUC, as the grid search baseline (and all are significantly more accurate than the validation AUC at initialization of gradient descent).
Also, we computed timings (Figure~\ref{fig:result-four-changepoint-folds}, top), and observed that the fastest method was the proposed line search (first min variant).
Interestingly, we observed that the slowest method overall was the linear variant, which stops the line search after $B$ iterations.
Using that method, each iteration of gradient descent is guaranteed to be log-linear $O(B\log B)$, but it takes more time overall because it must take more steps of gradient descent (each of which has a relatively small step size / learning rate).
Overall, these empirical results show that the proposed line search yields useful speedups relative to grid search, when learning a linear model to minimize AUM and maximize AUC.


\begin{figure}
    \centering
    \includegraphics[width=0.9\textwidth]{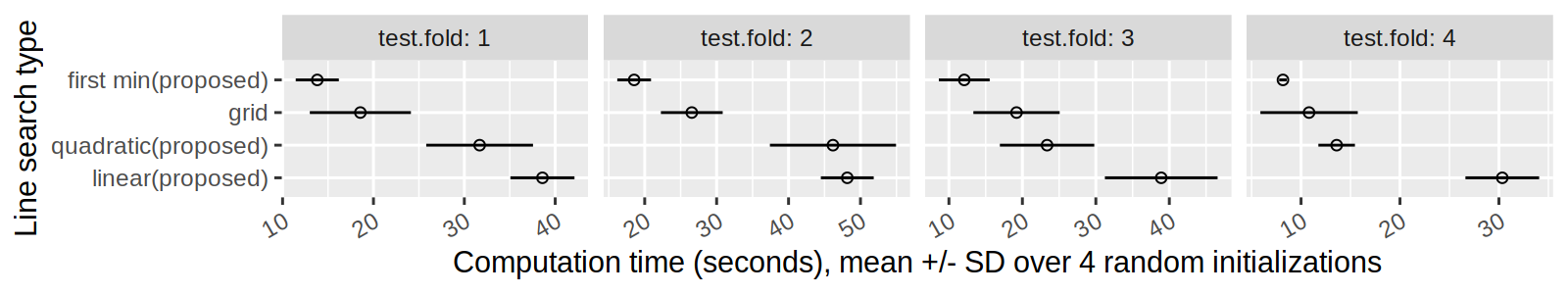}
    \\
    \includegraphics[width=0.9\textwidth]{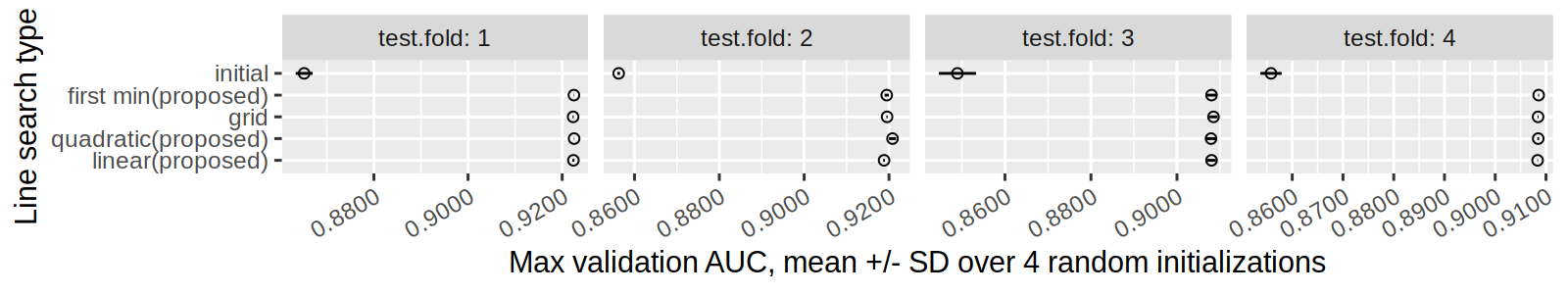}
    \vskip -0.4cm
    \caption{
    we compared the proposed line search to grid search, in terms of computation time (top) and max validation AUC (bottom), in four supervised change-point problems (test folds).
    It is clear that the first min method is consistently fastest, and has comparable values for max validation AUC.}
    \label{fig:result-four-changepoint-folds}
\end{figure}
\section{Discussion and conclusions}
\label{sec:discussion}


This paper proposed a new algorithm for efficient line search, for learning a linear model with gradient descent.
The proposed algorithm exploits the structure of the piecewise linear/constant AUM/AUC, in order to get a complete representation of those functions, for a range of step sizes, which can be used to pick the best learning rate in each step of gradient descent. 
For future work, we are interested in exploring extensions to neural networks with the ReLU activation function, which is piecewise linear, so could potentially be handled using a modification to our proposed algorithm.

\paragraph{Broader Impacts.}
The proposed algorithm could save time if applied to real binary classification or changepoint problems; negative impacts include potential misuse, similar to any algorithm.

\paragraph{Limitations.}
The proposed line search only works for a linear model.

\paragraph{Reproducible Research Statement.} 
All of the software/data to make the figures in this paper can be downloaded from a GitHub repository: \url{https://github.com/tdhock/max-generalized-auc}. 
We also provide a free/open-source C++ implementation of the proposed algorithm, as the function \verb|aum_line_search| in the \verb|aum| R package, on CRAN and \url{https://github.com/tdhock/aum}.

\bibliographystyle{abbrvnat}
\bibliography{refs}

\end{document}